\documentclass[a4paper]{article}
\usepackage{authblk}
\usepackage{amssymb}
\usepackage[sort&compress,square,numbers]{natbib}
\usepackage[titlenumbered,ruled,lined,linesnumbered,algo2e]{algorithm2e}
\usepackage{graphicx}
\usepackage{amsmath}
\usepackage{xspace}
\usepackage{bm}
\usepackage{amsthm}
\usepackage{afterpage}
\usepackage{indentfirst}
\usepackage{enumerate}
\usepackage{url}
\usepackage{comment}
\usepackage[scale={0.74,0.78}, hmarginratio=1:1, vmarginratio=1:1, headheight=2ex, headsep = 2ex, footskip = 3.6ex, nomarginpar]{geometry}
\usepackage{color}
\usepackage{lineno}
\usepackage{hyperref}

\newtheorem{theorem}{Theorem}
\newtheorem{lemma}[theorem]{Lemma}

\newtheorem{corollary}{Corollary}

\theoremstyle{definition}
\newtheorem{definition}{Definition}
\theoremstyle{remark}
\newtheorem{remark}{Remark}

\newcommand{\bz}{\ensuremath{\bm{z}}\xspace}
\newcommand{\emp}{\ensuremath{\ecal_{\bz}}\xspace}

\newcommand{\rbb}{\ensuremath{\mathbb{R}}\xspace}
\newcommand{\nbb}{\ensuremath{\mathbb{N}}\xspace}

\newcommand{\fcal}{\ensuremath{\mathcal{F}}\xspace}
\newcommand{\tfcal}{\ensuremath{\widetilde{\mathcal{F}}}\xspace}
\newcommand{\ncal}{\ensuremath{\mathcal{N}}\xspace}
\newcommand{\xcal}{\ensuremath{\mathcal{X}}\xspace}
\newcommand{\ycal}{\ensuremath{\mathcal{Y}}\xspace}
\newcommand{\zcal}{\ensuremath{\mathcal{Z}}\xspace}

\newcommand{\gcal}{\ensuremath{\mathcal{G}}\xspace}
\newcommand{\hcal}{\ensuremath{\mathcal{H}}\xspace}
\newcommand{\ecal}{\ensuremath{\mathcal{E}}\xspace}

\newcommand{\ocal}{\ensuremath{\mathcal{O}}\xspace}

\newcommand{\ebb}{\ensuremath{\mathbb{E}}\xspace}

\newcommand{\lrceil}[1]{\ensuremath{\left\lceil#1\right\rceil}\xspace}
\newcommand{\argmin}{\ensuremath{\mathop{\mathrm{argmin}}}\xspace}%
\newcommand{\dif}{\ensuremath{\mathrm{d}}\xspace}%

\newcommand{\lrbrack}[1]{\ensuremath{\left[#1\right]}\xspace}
\newcommand{\lrgroup}[1]{\ensuremath{\left(#1\right)}\xspace}
\newcommand{\abs}[1]{\ensuremath{\left|#1\right|}\xspace}

\newcommand{\pr}[1]{\ensuremath{\text{Pr}\big\{#1\big\}}\xspace}

\newcommand{\lto}{\ensuremath{\longrightarrow}\xspace}

\numberwithin{equation}{section}

\begin{document}

\title{Local Rademacher Complexity Bounds based on Covering Numbers}
\author[1]{Yunwen Lei\thanks{yunwen.lei@hotmail.com. Part of the work was done at Wuhan University}}
\author[2]{Lixin Ding\thanks{lxding@whu.edu.cn}}
\author[3]{Yingzhou Bi\thanks{byzhou@163.com}}

\affil[1]{Department of Mathematics, City University of Hong Kong}
\affil[2]{State Key Lab of Software Engineering, School of Computer, Wuhan University}
\affil[3]{Science Computing and Intelligent Information Processing of GuangXi Higher Education Key Laboratory, Guangxi Teachers Education University}

\date{}
\maketitle

\begin{abstract}
This paper provides a general result on controlling local Rademacher complexities, which captures in an elegant form to relate the complexities with constraint on the expected norm to the corresponding ones with constraint on the empirical norm. This result is convenient to apply in real applications and could yield refined local Rademacher complexity bounds for function classes satisfying general entropy conditions. We demonstrate the power of our complexity bounds by applying them to derive effective generalization error bounds.

\end{abstract}


\section{Introduction}
Machine learning refers to a process of inferring the underlying relationship among input-output variables from a previously chosen hypothesis class \hcal, on the basis of some scattered, noisy examples~\citep{hastie2001elements,vapnik2000nature}. Generalization analysis on learning algorithms stands a central place in machine learning since it is important to understand the factors influencing models' behavior, as well as to suggest ways to improve them~\citep{bousquet2003new,bartlett2005local,cortes2013learning,chen2014extreme,blanchard2008statistical,lei2015multi}. One seminar example can be found in the \emph{multiple kernel learning} (MKL) context, where \citet{cortes2013learning} established a framework showing how the generalization analysis in \citep{kloft2011local,mendelson2003performance,kloft2012convergence} could motivate two novel MKL algorithms.


\citet{vapnik1971uniform} pioneered the research on learning theory by relating generalization errors to the supremum of an empirical process: $\sup_{f\in\fcal}[Pf-P_nf]$, where \fcal is the associated loss class induced from the hypothesis space, $P$ and $P_n$ are the true probability measure and the empirical probability measure, respectively. It was then indicated that this supremum is closely connected with the ``size'' of the space \fcal~\citep{vapnik2000nature,vapnik1971uniform}. For a finite class of functions, its size can be simply measured by its cardinality. ~\citet{vapnik2000nature} provided a novel concept called VC dimension to characterize the complexity of $\{0,1\}$-valued function classes, by noticing that the quantity of significance is the number of points acquired when projecting the function class onto the sample. Other quantities like covering numbers, which measure the number of balls required to cover the original class, have been introduced to capture, on a finer scale, the ``size'' of real-valued function classes~\citep{zhou2002covering,cucker2002mathematical,zhou2003capacity,kolmogorov1959varepsilon}. With the recent development in concentration inequalities and empirical process theory, it is possible to obtain a slightly tighter estimate on the ``size'' of \hcal through the remarkable concept called Rademacher complexity~\citep{bartlett2002rademacher,bartlett2005local,ying2010rademacher,koltchinskii2001rademacher}.

However, all the above mentioned approaches provide only global estimates on the complexity of function classes, and they do not reflect how a learning algorithm explores the function class and interacts with the examples~\citep{bousquet2003new,bousquet2002concentration}. Moreover, they are bound to control the deviation of empirical errors from the true errors simultaneously over the whole class, while the quantity of primary importance is only that deviation for the particular function picked by the learning algorithm, which may be far from reaching this supremum~\citep{bartlett2005local,koltchinskii2000rademacher,oneto2015local}. Therefore, the analysis based on a global complexity would give a rather conservative estimate. On the other hand, most learning algorithms are inclined towards choosing functions possessing small empirical errors and hopefully also small generalization errors~\citep{bousquet2003new}. Furthermore, if there holds a relationship between variances and expectations like $\text{Var}(f)\leq B(Pf)^\alpha$, these functions will also admit small variances. That is to say, the obtained prediction rule is likely to fall into a subclass with small variances~\citep{bartlett2005local}. Due to the seminar work of~\citet{koltchinskii2000rademacher} and~\citet{massart2000some}, it turns out that the notion of Rademacher complexity can be naturally modified to take this into account, yielding the so-called local Rademacher complexity~\citep{koltchinskii2000rademacher}. Since local Rademacher complexity is always smaller than the global counterpart, the discussion based on local Rademacher complexities always yields significantly better learning rates under the variance-expectation conditions.


\citet{mendelson2003few,mendelson2002improving} initiated the discussion of estimating local Rademacher complexities with covering numbers and these complexity bounds are very effective in establishing \emph{fast} learning rates. However, the discussions in \citep{mendelson2003few,mendelson2002improving} are somewhat dispersed in the sense that the author did not provided a general result applicable to all function classes. Indeed, \citet{mendelson2003few,mendelson2002improving} derived local Rademacher complexity bounds for several function classes satisfying different entropy conditions case-by-case, and the involved deduction also relies on the specific entropy conditions. \citet{mendelson2003performance} also derived, for a general \emph{Reproducing Kernel Hilbert Space} (RKHS), an interesting local Rademacher complexity bound based on the eigenvalues of the associated integral operator, which was later generalized to $\ell_p$-norm MKL context~\citep{kloft2011local,kloft2012convergence,lv2015optimal}. These results are exclusively developed for RKHSs and it still remains unknown whether they could be extended to general function classes. In this paper, we try to refine these discussions by providing some general and sharp results on controlling local Rademacher complexities by covering numbers. A distinguished property of our result is that it captures in an elegant form to relate local Rademacher complexities to the associated \emph{empirical} local Rademacher complexities, which allows us to improve the existing local Rademacher complexity bounds for function classes with different entropy conditions in a systematic manner.  We also demonstrate the effectiveness of these complexity bounds by applying them to refine the existing learning rates.

The paper is organized as follows. Section~\ref{sec:statement} formulates the problem. Section~\ref{sec:local-rademacher complexity} provides a general local Rademacher complexity bound as well as its applications to different function classes.  Section~\ref{sec:application} applies our complexity bounds to generalization analysis. All proofs are presented in Section \ref{sec:proof}. Some conclusions are presented in Section~\ref{sec:conclusion}.

\section{Statement of the problem\label{sec:statement}}
We first introduce some notations which will be used throughout this paper. For a measure $\mu$ and a positive number $1\leq q<\infty$, the notation $L_q(\mu)$ means the collection of functions for which the norm $\|f\|_{L_q(\mu)}:=(\int|f|^q\dif\mu)^{1/q}$ is finite. For a class \fcal of functions, we use the abbreviation $a\fcal:=\{af:f\in\fcal\}$, and denote by
\begin{equation}\label{minus-class}
  \tfcal:=\{f-g:f,g\in\fcal\}
\end{equation}
the class consisting of those elements which can be represented as the minus of two elements in \fcal.
For a real number $a$, $\lrceil{a}$ indicates the least integer not less than $a$, and $\log a$ represents the natural logarithm of $a$. By $c(\cdot)$ we denote any quantity of a constant multiple of the involved arguments and its exact value may change from line to line, or even within the same line.

\begin{definition}[Empirical measure]
  Let $S$ be a set and let $s_1,s_2,\ldots,s_n$ be $n$ points in $S$, then the empirical measure $P_n$ supported on $s_1,s_2,\ldots,s_n$ is defined as
  \begin{equation}\label{empirical-measure}
    P_n(A):=\frac{1}{n}\sum_{i=1}^n\chi_A(s_i),\qquad \text{for any }A\subset S,
  \end{equation}
  where $\chi_I$ is the characteristic function defined by $\chi_A(s)=0$ if $s\not\in A$ and $\chi_A(s)=1$ if $s\in A$.
\end{definition}

If $Q$ is a measure and $f$ is a measurable function, it is convenient~\citep{bousquet2003new} to use the notation $Qf=\int f\dif Q=\ebb f$. Now, for the empirical measure $P_n$ supported on $Z_1,\ldots,Z_n$, the empirical average of $f$ can be abbreviated as $P_nf=\frac{1}{n}\sum_{i=1}^nf(Z_i)$.
\begin{definition}[Covering number~\citep{kolmogorov1959varepsilon}]\label{def:covering-number}
  Let $(\gcal,d)$ be a metric space and set $\fcal\subseteq \gcal$. For any $\epsilon>0$, a set $\fcal^\triangle$ is called an $\epsilon$-cover of \fcal if for every $f\in\fcal$ we can find an element $g\in\fcal^\triangle$ satisfying $d(f,g)\leq\epsilon$. An $\epsilon$-cover $\fcal^\triangle$ is called a proper $\epsilon$-cover if $\fcal^\triangle\subseteq\fcal$. The covering number $\ncal(\epsilon,\fcal,d)$ is the cardinality of a minimal proper $\epsilon$-cover of \fcal, that is
  $$\ncal(\epsilon,\fcal,d):=\min\{|\fcal^\triangle|:\fcal^\triangle\subseteq\fcal \text{ is an $\epsilon$-cover of \fcal}\}.$$We also define the logarithm of covering number as the entropy number.
\end{definition}

For brevity, when \gcal is a normed space with norm $\|\cdot\|$, we also denote by $\ncal(\epsilon, \fcal, \|\cdot\|)$ the covering number of \fcal with respect to the metric $d(f,g):=\|f-g\|$. Introduce the notation:
\begin{equation}\label{metric-capacity}
  \ncal(\epsilon,\fcal,\|\cdot\|_p):=\sup\nolimits_n\sup\nolimits_{P_n}\ncal(\epsilon,\fcal,\|\cdot\|_{L_p(P_n)}).
\end{equation}
%

\begin{definition}[Rademacher complexity~\citep{bartlett2002rademacher}]
  Let $P$ be a probability measure on \xcal from which the examples $X_1,\ldots,X_n$ are independently drawn. Let $\sigma_1,\ldots,\sigma_n$ be independent Rademacher random variables that have equal probability of being $1$ or $-1$. For a class \fcal of functions $f:\xcal\to\rbb$, introduce the notations:\[R_nf=\frac{1}{n}\sum_{i=1}^n\sigma_if(X_i),\qquad R_n\fcal=\sup_{f\in\fcal}R_nf.\]The Rademacher complexity $\ebb R_n\fcal$ and empirical Rademacher complexity $\ebb_\sigma R_n\fcal$ are defined by\[\ebb R_n\fcal:=\ebb\lrbrack{\sup_{f\in\fcal}\frac{1}{n}\sum_{i=1}^n\sigma_if(X_i)},\qquad\ebb_\sigma R_n\fcal:=\ebb\left[\sup_{f\in\fcal}\frac{1}{n}\sum_{i=1}^n\sigma_if(X_i)\bigg|X_1,\ldots,X_n\right].\]
\end{definition}

In this paper we concentrate our attention on local Rademacher complexities. The word local means that the class over which the Rademacher process is defined is a subset of the original class. We consider here local Rademacher complexities of the following form:
\[\ebb R_n\{f\in\fcal:Pf^2\leq r\}\qquad\text{or}\qquad\ebb_\sigma R_n\{f\in\fcal:P_nf^2\leq r\}.\]We refer to the former as the local Rademacher complexity and the latter as the \emph{empirical} local Rademacher complexity. The parameter $r$ is used to filter out those functions with large variances~\citep{mendelson2003performance}, which are of little significance in the learning process since learning algorithms are unlikely to pick them.

\section{Estimating local Rademacher complexities}\label{sec:local-rademacher complexity}
This section is devoted to establishing a general local Rademacher complexity bound. For this purpose, we first show how to control \emph{empirical} local Rademacher complexities. The empirical radii are then connected with the true radii via the contraction property of Rademacher averages~(Lemma~\ref{lem:contraction inequality}). Some examples illustrating the power of our result are also presented.

\subsection{Local Rademacher complexity bounds}
\citet{mendelson2002improving,mendelson2003few} studied $\ebb R_n\{f\in\fcal:Pf^2\leq r\}$ by relating it with
\begin{equation}\label{empirical-radius}
  \ebb R_n\{f\in\fcal:P_nf^2\leq\hat{r}\},\quad \hat{r}:=\sup_{f\in\fcal:Pf^2\leq r}P_nf^2,
\end{equation}
the latter of which involves an empirical radius defined w.r.t. the empirical measure $P_n$ and can be further tackled by standard entropy integral \citep{dudley1967sizes}, yielding a bound of the following form:
\begin{equation}\label{mendelson-lrc}
  \ebb R_n\{f\in\fcal:Pf^2\leq r\}\leq c\cdot \ebb\int_0^{\hat{r}}\log^{\frac{1}{2}}\ncal(\epsilon,\fcal,\|\cdot\|_{L_2(P_n)})d\epsilon.
\end{equation}
Although the expectation $\ebb\sqrt{\hat{r}}$ can be controlled by $r$ plus the local Rademacher complexity itself~\citep{ledoux1991probability}
\begin{equation}\label{empirical-radius-bound}
  \ebb\sqrt{\hat{r}}\leq r+4\sup_{f\in\fcal}\|f\|_\infty\ebb R_n\{f\in\fcal:Pf^2\leq r\},
\end{equation}
it is generally not trivial to control the integral in Eq. \eqref{mendelson-lrc} since the random variable $\hat{r}$ appears in the upper limit of the integral (the bound Eq. \eqref{empirical-radius-bound} can not be trivially used to control the r.h.s. of Eq. \eqref{mendelson-lrc}). Mendelson's \citep{mendelson2003few,mendelson2002improving} idea is, under different entropy conditions, to construct different upper bounds on the involved integral for which the random variable $\hat{r}$ appears in a relatively simple term. For example, for the function class \fcal satisfying $\log\ncal(\epsilon,\fcal,\|\cdot\|_2)\leq \log^p\frac{\gamma}{\epsilon}$, \citet{mendelson2003few} established the following bound on the integral:
\begin{equation}\label{mendelson-example-1}
  \ebb\int_0^{\hat{r}}\log^{\frac{1}{2}}\ncal(\epsilon,\fcal,\|\cdot\|_{L_2(P_n)})d\epsilon\leq \ebb\int_0^{\sqrt{\hat{r}}}\log^{\frac{p}{2}}\frac{\gamma}{\epsilon}d\epsilon\leq 2\ebb\Big[\sqrt{\hat{r}}\log^{\frac{p}{2}}\frac{c(p,\gamma)}{\sqrt{\hat{r}}}\Big].
\end{equation}
The term $\sqrt{\hat{r}}\log^{\frac{p}{2}}\frac{c(p,\gamma)}{\sqrt{\hat{r}}}$ turns out to be concave w.r.t. $\sqrt{\hat{r}}$, which, together with Jensen's inequality, can be controlled by applying the standard upper bound \eqref{empirical-radius-bound}. Although these deductions are elegant, they do not allow for general bounds for local Rademacher complexities, and sometimes yield unsatisfactory results due to the looseness introduced by constructing an additional artificial upper bound for the integral in Eq. \eqref{mendelson-lrc} (e.g., Eq. \eqref{mendelson-example-1}).

We overcome these drawbacks by providing a general result on controlling local Rademacher complexity bounds. The step stone is the following lemma controlling local Rademacher complexity on a sub-class involving a random radius $\hat{r}$ by a local Rademacher complexity on a sub-class involving a deterministic and adjustable parameter $\epsilon$ plus a linear function of $\sqrt{\hat{r}}$, which allows for a direct use of the standard upper bound on $\ebb\sqrt{\hat{r}}$ and excludes the necessity of constructing non-trivial bounds for the integral in Eq. \eqref{mendelson-lrc}. Our basic strategy, analogous to \citep{lei2014refined,srebro2010optimistic,lei2015generalization},  is to approximate the original function class \fcal with an $\epsilon$-cover, thus relating the local Rademacher complexity of \fcal to that of two related function classes. One class is of finite cardinality and can be approached by the Massart lemma (Lemma~\ref{lem:massart}), while the other is of small magnitude and is defined by empirical radii.


\begin{lemma}\label{lem:empirical-local-rademacher-L2}
  Let \fcal be a function class and let $P_n$ be the empirical measure supported on the points $X_1,\ldots,X_n$, then we have the following complexity bound ($r$ can be stochastic w.r.t. $X_i$, a typical choice of $r$ is the term $\hat{r}$ defined in Eq. \eqref{empirical-radius}):
  \[\ebb_\sigma R_n\{f\in\fcal:P_nf^2\leq r\}\leq\inf_{\epsilon>0}\lrbrack{\ebb_\sigma R_n\{f\in\tfcal:P_nf^2\leq\epsilon^2\}+\sqrt{\frac{2r\log\ncal(\epsilon/2,\fcal,\|\cdot\|_{L_2(P_n)})}{n}}}.\]
\end{lemma}



\begin{theorem}[Main theorem]\label{thm:local-rademacher-L2}
  Let \fcal be a function class satisfying $\|f\|_\infty\leq b, \forall f\in\fcal$. There holds the following inequality:
  \begin{multline}\label{local-rademacher-L2}
    \ebb R_n\{f\in\fcal:Pf^2\leq r\}\leq\inf_{\epsilon>0}\Bigg[2\ebb R_n\{f\in\tfcal:P_nf^2\leq \epsilon^2\}+\\\frac{8b\log\ncal(\epsilon/2,\fcal,\|\cdot\|_2)}{n}+\sqrt{\frac{2r\log\ncal(\epsilon/2,\fcal,\|\cdot\|_2)}{n}}\Bigg].
  \end{multline}
\end{theorem}

\begin{remark}
  An advantage of Theorem \ref{thm:local-rademacher-L2} over the existing local Rademacher complexity bounds consists in the fact that it provides a general framework for controlling local Rademacher complexities, from which, as we will show in Section \ref{subsec:examples}, one can trivially derive explicit local Rademacher complexity bounds when the entropy information is available. Furthermore, since Theorem \ref{thm:local-rademacher-L2} does not involve an artificial upper bound for the integral in Eq. \eqref{mendelson-lrc} (e.g., Eq. \eqref{mendelson-example-1}) , it could yield sharper local Rademacher complexity bounds (see Remark \ref{rem:case-study-1}, \ref{rem:case-study-3}, \ref{rem:comparison-complexity-2}) when compared to the results in \citep{mendelson2003few,mendelson2002improving}.\hfill\qed
\end{remark}

\subsection{Some examples\label{subsec:examples}}
We now demonstrate the effectiveness of Theorem~\ref{thm:local-rademacher-L2} by applying it to some interesting classes satisfying general entropy conditions. Our discussion is based on the refined entropy integral~\eqref{refined-dudley-integral}, which can be used to tackle the situation where the standard entropy integral~\citep{dudley1967sizes} diverges.

\begin{corollary}\label{cor:local-rademacher-1}
  Let \fcal be a function class with $\sup_{f\in\fcal}\|f\|_\infty\leq b$. Assume that there exist three positive numbers $\gamma,d,p$ such that $\log\ncal(\epsilon,\fcal,\|\cdot\|_2)\leq d\log^p(\gamma/\epsilon)$ for any $0<\epsilon\leq\gamma$, then for any $0<r\leq\gamma^2$ and $n\geq\gamma^{-2}$ there holds that
  \begin{multline*}
  \ebb R_n\{f\in\fcal:Pf^2\leq r\}\leq c(b,p,\gamma)\min\Bigg[\Big(\sqrt{\frac{dr\log^p(2\gamma r^{-1/2})}{n}}\!+\!\frac{d\log^p(2\gamma r^{-1/2})}{n}\Big),\\ \Big(\frac{d\log^p(2\gamma n^{1/2})}{n}\!+\!\sqrt{\frac{rd\log^p(2\gamma n^{1/2})}{n}}\Big)\Bigg].
  \end{multline*}
\end{corollary}

\begin{remark}\label{rem:case-study-1}
  For function classes \fcal meeting the condition of Corollary~\ref{cor:local-rademacher-1}, \citet[Lemma 2.3]{mendelson2002improving} derived the following complexity bound
  \begin{equation}\label{cor-mendelson-1}
    \ebb R_n\{f\in\fcal:Pf^2\leq r\}\leq c(b,p,\gamma)\max\lrbrack{\frac{d}{n}\log^p\frac{1}{\sqrt{r}},\sqrt{\frac{dr}{n}}\log^{p/2}\frac{1}{\sqrt{r}}}.
  \end{equation}
  It is interesting to compare the bound~\eqref{cor-mendelson-1} with ours and the difference can be seen in the following three aspects:
  \begin{enumerate}[(1)]
    \item Firstly, it is obvious that the r.h.s. of Eq.~\eqref{cor-mendelson-1} is of the same order of magnitude to $\sqrt{drn^{-1}\log^p(r^{-1/2})}+dn^{-1}\log^p(r^{-1/2})$. Consequently, our bound can be no worse than Eq.~\eqref{cor-mendelson-1}.
    \item Furthermore, as we will see in Section~\ref{sec:application}, the upper bound in Eq.~\eqref{cor-mendelson-1} is not a sub-root function, which adds some additional difficulty in applying it to the generalization analysis. As a comparison, the upper bound $dn^{-1}\log^p( n^{1/2})+\sqrt{rdn^{-1}\log^p(n^{1/2})}$ satisfies the sub-root condition (see definition of sub-root functions in Section \ref{sec:application}) and thus can be convenient to use in the generalization analysis.
    \item Thirdly, Eq.~\eqref{cor-mendelson-1} is not consistent with the natural opinion on what the complexity bound should be. For example, when $r$ approaches to $0$ it is expected that the term $\ebb R_n\{f\in\fcal:Pf^2\leq r\}$ should monotonically decrease to a limiting point. However, the upper bound in Eq.~\eqref{cor-mendelson-1} diverges to $\infty$ as $r\to 0$. As a comparison, our result does not violate such consistence since the term $dn^{-1}\log^p(n^{1/2})+\sqrt{rdn^{-1}\log^p(n^{1/2})}$ is always an increasing function of $r$.  \hfill\qed
  \end{enumerate}
\end{remark}

\begin{corollary}\label{cor:local-rademacher-3}
  Let \fcal be a function class with $\sup_{f\in\fcal}\|f\|_\infty\leq b$. Assume that there exist two constants $\gamma>0,p>0$ such that
  \begin{equation}\label{entropy-condition-3}
    \log\ncal(\epsilon,\fcal,\|\cdot\|_2)\leq\gamma\epsilon^{-p}\log^2\frac{2}{\epsilon},
  \end{equation}
  then we have the following complexity bound:
  \begin{equation}\label{cor-local-rademacher-3}
  \ebb R_n\{f\in\fcal:Pf^2\leq r\}\leq\begin{cases}
    c\inf\limits_{\epsilon>0}\lrbrack{n^{-1/2}\epsilon^{1-p/2}\log\frac{1}{\epsilon}+\epsilon^{-p}n^{-1}\log^2\frac{4}{\epsilon}+\sqrt{r\epsilon^{-p}n^{-1}\log^2\frac{4}{\epsilon}}}&\text{if }0<p<2,\\[0.6mm]
    c\big[n^{-1/2}\log^2n+\sqrt{rn^{-1}}\big]&\text{if }p=2,\\[0.6mm]
    c[n^{-1/p}\log n+\sqrt{rn^{-1}}]&\text{if }p>2,
  \end{cases}
  \end{equation}
  where $c:=(b,p,\gamma)$ is a constant dependent on $b,p$ and $\gamma$.
\end{corollary}

\begin{remark}\label{rem:case-study-3}
  We now compare Corollary \ref{cor:local-rademacher-3} with the following inequality established in \citep[Eq. (3.5)]{mendelson2003few} under the entropy condition \eqref{entropy-condition-3} with $0<p<2$:
  \begin{equation}\label{rem-eq:comparison-complexity-3}
  \ebb R_n\{f\in\fcal:Pf^2\leq r\}\leq c(b,p,\gamma)(n^{-2/(p+2)}\log^{\frac{4}{2+p}}\frac{2}{r}+n^{-1/2}r^{(2-p)/4}\log\frac{2}{r}),\qquad0<p<2.
  \end{equation}
  The upper bound in Eq. \eqref{rem-eq:comparison-complexity-3} is not a sub-root function. Furthermore, our bound grows monotonically increasing w.r.t. $r$, while the bound \eqref{rem-eq:comparison-complexity-3} diverges to $\infty$ as $r\to0$, which violates the natural property the local Rademacher complexity should admit. \hfill\qed
\end{remark}

\begin{corollary}\label{cor:local-rademacher-2}
  Let \fcal be a function class with $\sup_{f\in\fcal}\|f\|_\infty\leq b$. Assume that there exist two constants $\gamma>0,p>0$ such that $\log\ncal(\epsilon,\fcal,\|\cdot\|_2)\leq\gamma\epsilon^{-p}$, then we have the following complexity bound:
  \begin{equation}\label{cor-local-rademacher-2}
  \ebb R_n\{f\in\fcal:Pf^2\leq r\}\leq\begin{cases}
    c(b,p,\gamma)\inf\limits_{\epsilon>0}\lrbrack{n^{-1/2}\epsilon^{1-p/2}+\epsilon^{-p}n^{-1}+\sqrt{r\epsilon^{-p}n^{-1}}}&\text{if }0<p<2,\\[0.6mm]
    c(b,p,\gamma)\lrbrack{n^{-1/2}\log n+\sqrt{r}n^{-1/2}}&\text{if }p=2,\\[0.6mm]
    c(b,p,\gamma)\lrbrack{n^{-1/p}+\sqrt{r}n^{-1/2}}&\text{if }p>2.
  \end{cases}
  \end{equation}
\end{corollary}
\begin{remark}\label{rem:comparison-complexity-2}
  As compared with the following inequality established in \citep[Eq. (3.4)]{mendelson2003few}
  \begin{equation}\label{rem-eq:comparison-complexity-2}
  \ebb R_n\{f\in\fcal:Pf^2\leq r\}\leq c(b,p,\gamma)(n^{-2/(p+2)}+n^{-1/2}r^{(2-p)/4}),\qquad0<p<2,
  \end{equation}
  Corollary~\ref{cor:local-rademacher-2} generalizes Eq.~\eqref{rem-eq:comparison-complexity-2} to the case $p\geq2$ on the one hand, and on the other hand provides a competitive result for the case $p<2$. For example, when $r\leq n^{-2/(p+2)}$ one can take $\epsilon=n^{-1/(p+2)}$ in Eq.~\eqref{cor-local-rademacher-2} to show that\[\ebb R_n\{f\in\fcal:Pf^2\leq r\}\leq c(b,p,\gamma)\lrbrack{n^{-2/(p+2)}+\sqrt{r}n^{-1/(p+2)}},\]which is no larger than Eq.~\eqref{rem-eq:comparison-complexity-2} since $\sqrt{r}n^{-1/(p+2)}\leq n^{-1/2}r^{(2-p)/4}$ for such $r$. Furthermore, for the case $r>n^{-2/(p+2)}$ one can also choose $\epsilon=r^{1/2}$ in Eq.~\eqref{cor-local-rademacher-2} to obtain that \[\ebb R_n\{f\in\fcal:Pf^2\leq r\}\leq c(b,p,\gamma)\lrbrack{n^{-1/2}r^{(2-p)/4}+r^{-p/2}n^{-1}},\] which is again no larger than Eq.~\eqref{rem-eq:comparison-complexity-2} since $r^{-p/2}n^{-1}\leq n^{-2/(p+2)}$ in this case. Therefore, our result is competitive to Eq.~\eqref{rem-eq:comparison-complexity-2} for any $r>0$. \hfill\qed
\end{remark}

\section{Applications to generalization analysis}\label{sec:application}
We now show how to apply the previous local Rademacher complexity bounds to study the generalization performance for learning algorithms. In the learning context, we are given an input space \xcal and an output space \ycal, along with a probability measure $P$ on $\zcal:=\xcal\times\ycal$. Given a sequence of examples $Z_1=(X_1,Y_1),\ldots,Z_n=(X_n,Y_n)$ independently drawn from $P$, our goal is to find a prediction rule (model) $h:\xcal\to\ycal$ to perform prediction as accurately as possible. The error incurred from using $h$ to do the prediction on an example $Z=(X,Y)$ can be quantified by a non-negative real-valued loss function $\ell(h(X,Y))$. The generalization performance of a model $h$ can be measured by its generalization error~\citep{wu2006learning,cucker2007learning} $\ecal(h):=\int\ell(h(X),Y)\dif P$. Since the measure $P$ is often unknown to us, the \emph{Empirical Risk Minimization} principle firstly establishes the so-called empirical error $\emp(h):=\frac{1}{n}\sum_{i=1}^n\ell(h(X_i),Y_i)$ to approximate $\ecal(h)$, and then searches the prediction rule $\hat{h}_n$ by minimizing $\emp(h)$ over a specified class \hcal called hypothesis space. That is, $\hat{h}_n:=\argmin_{h\in\hcal}\emp(h)$. Denoting by $h^*:=\argmin_{h\in\hcal}\ecal(h)$ the best prediction rule attained in \hcal, generalization analysis aims to relate the excess generalization error $\ecal(\hat{h}_n)-\ecal(h^*)$ to the empirical behavior of $\hat{h}_n$ over the sample. 

Our generalization analysis is based on Theorem~\ref{thm:bartlett} in~\citet{bartlett2005local}, which justifies the use of the Rademacher complexity associated with a small subset of the original class as a complexity term in an error bound. We call a function $\psi:[0,\infty)\lto[0,\infty)$ sub-root if it is nonnegative, nondecreasing and if $r\lto\psi(r)/\sqrt{r}$ is nonincreasing for $r>0$. If $\psi$ is a sub-root function, then it can be checked~\citep{bartlett2005local,blanchard2008statistical} that the equation $\psi(r)=r$ has a unique positive solution $r^*$, which is referred to as the fixed point of $\psi$.
\begin{lemma}[\citep{bartlett2005local}]\label{thm:bartlett}
  Let \fcal be a class of functions taking values in $[a,b]$ and assume that there exist some functional $T:\fcal\lto\rbb^+$ and some constant $B$ such that $\text{Var}(f)\leq T(f)\leq BPf$ for every $f\in\fcal$. Let $\psi$ be a sub-root function with the fixed point $r^*$. If for any $r\geq r^*$, $\psi$ satisfies\[\psi(r)\geq B\ebb R_n\{f\in\fcal:T(f)\leq r\},\]then for any $K>1$ and any $t>0$, the following inequality holds with probability at least $1-e^{-t}$:
  \begin{equation}\label{bartlett}
    Pf\leq\frac{K}{K-1}P_nf+\frac{704K}{B}r^*+\frac{t(11(b-a)+26BK)}{n},\qquad \forall f\in\fcal.
  \end{equation}
\end{lemma}

\begin{theorem}\label{thm:application-1}
  Let \hcal be the hypothesis space and
  $$
    \fcal:=\{Z=(X,Y)\to\ell(h(X),Y)-\ell(h^*(X),Y):h\in\hcal\}
  $$
  be the shifted loss class. Suppose that $\ell$ is $L$-Lipschitz, $\sup_{h\in\hcal}\|h\|_\infty\leq b,\pr{|Y|\leq b}=1$ and there exist three positive constants $\gamma, d$ and $p$ satisfying $\log\ncal(\epsilon,\hcal,\|\cdot\|_2)\leq d\log^p(\gamma/\epsilon)$. Suppose the variance-expectation condition holds for functions in $\fcal$, i.e., there exists a constant $B>0$ such that $Pf^2\leq BPf,\forall f\in\fcal$. Then, for any $0<\delta<1$, $\hat{h}_n$ satisfies the following inequality with probability at least $1-\delta$:\[\ecal(\hat{h}_n)-\ecal(h^*)\leq c\lrbrack{\frac{d\log^pn}{n}+\frac{\log(1/\delta)}{n}},\]where $c$ is a constant depending on $B,p,\gamma,b$ and $L$.
\end{theorem}

\begin{remark}
  It is possible to derive generalization error bounds using the local Rademacher complexity bounds given in \citep{mendelson2003few} (Eq. \eqref{cor-mendelson-1}) under the same entropy condition. An obstacle in the way of applying Lemma \ref{thm:bartlett} is that the r.h.s. of Eq. \eqref{cor-mendelson-1} is not a sub-root function. The trick towards this problem is to consider the local Rademacher complexity of a slightly larger function class (the star-shaped space, or star-hull, $\text{star}(\fcal):=\{\alpha f:f\in\fcal,\alpha\in[0,1]\}$ of $\fcal$), which always satisfies the sub-root property and can be related to the original class by the following inequality due to \citet[Lemma 3.9]{mendelson2003few}:
  $$
    \log\ncal(2\epsilon,\text{star}(\fcal),\|\cdot\|_2)\leq\log\frac{2}{\epsilon}+\log\ncal(\epsilon,\fcal,\|\cdot\|_2).
  $$
  With this trick and plugging Eq. \eqref{cor-mendelson-1} into Lemma \ref{thm:bartlett}, one can derive the following generalization bound with probability at least $1-\delta$:
  $$
    \ecal(\hat{h}_n)-\ecal(h^*)\leq c\lrbrack{\frac{d\log^{\max(1,p)}n}{n}+\frac{\log(1/\delta)}{n}},
  $$
  which is slightly worse than the bound in Theorem \ref{thm:application-1} for $p<1$. Furthermore, notice that our upper bound on local Rademacher complexities is always a sub-root function, which is more convenient to use in Lemma \ref{thm:bartlett} and does not require the trick of introducing an additional star-hull.
\end{remark}

\begin{theorem}\label{thm:application-2}
  Under the same condition of Theorem~\ref{thm:application-1} except the entropy condition Eq. \eqref{entropy-condition-3}, the following inequality holds with probability at least $1-\delta$:
  $$
    \ecal(\hat{h}_n)-\ecal(h^*)\leq c(n^{-\frac{p}{p+2}}(\log n)^{\frac{2-p}{p+2}}\log\frac{n}{(\log n)^{\frac{2}{p+2}}}+n^{-1}\log(1/\delta)),
  $$
  where $c$ is a constant depending on $B,p,\gamma,b$ and $L$.
\end{theorem}

\begin{remark}
  Since the local Rademacher complexity bound given in Eq. \eqref{rem-eq:comparison-complexity-3} is not sub-root, the application of it to study generalization performance also requires the trick of star-hull argument. Indeed, with this trick one can show that the bound \eqref{rem-eq:comparison-complexity-3} could yield the following generalization guarantee with probability at least $1-\delta$:
  $$
    \ecal(\hat{h}_n)-\ecal(h^*)\leq c(n^{-\frac{p}{p+2}}(\log n)^{\frac{4}{p+2}}+n^{-1}\log(1/\delta)),
  $$
  which is slightly worse than the bound given in Theorem \ref{thm:application-2}.
\end{remark}

\section{Proofs\label{sec:proof}}
\subsection{Proofs on general local Rademacher complexity bounds}
\begin{proof}[Proof of Lemma \ref{lem:empirical-local-rademacher-L2}]
  For a temporarily fixed $\epsilon>0$, let $\fcal^{\triangle}$ be a minimal proper $\epsilon$-cover of the class $\{f\in\fcal:P_nf^2\leq r\}$ with respect to the metric $\|\cdot\|_{L_2(P_n)}$. According to the definition of covering numbers, we know that $\fcal^\triangle\subseteq\{f\in\fcal:P_nf^2\leq r\}$. Furthermore, Lemma~\ref{lem:covering-number-relationship} shows that $|\fcal^\triangle|\leq\ncal(\epsilon/2,\fcal,\|\cdot\|_{L_2(P_n)})$. For any $f\in\fcal$, let $f^\triangle$ be an element of $\fcal^\triangle$ satisfying $\|f-f^\triangle\|_{L_2(P_n)}\leq\epsilon$. Then, we have
  \begin{equation}\label{empirical-local-rademacher-L2-1}
  \begin{split}
    R_n\{f\in\fcal:P_nf^2\leq r\}&=\sup_{\{f\in\fcal:P_nf^2\leq r\}}\lrbrack{\frac{1}{n}\sum_{i=1}^n\sigma_if(X_i)-\frac{1}{n}\sum_{i=1}^n\sigma_if^\triangle(X_i)+\frac{1}{n}\sum_{i=1}^n\sigma_if^\triangle(X_i)}\\
    &\hspace*{-1.5cm}\leq\sup_{\{f\in\fcal:P_nf^2\leq r\}}\frac{1}{n}\sum_{i=1}^n\sigma_i[f(X_i)-f^\triangle(X_i)]+\sup_{\{f\in\fcal:P_nf^2\leq r\}}\frac{1}{n}\sum_{i=1}^n\sigma_if^\triangle(X_i)\\
    &\hspace*{-1.5cm}\leq\sup_{\{f\in\fcal:P_nf^2\leq r\}}\frac{1}{n}\sum_{i=1}^n\sigma_i[f(X_i)-f^\triangle(X_i)]+\sup_{\{f\in\fcal^\triangle:P_nf^2\leq r\}}\frac{1}{n}\sum_{i=1}^n\sigma_if(X_i),
  \end{split}
  \end{equation}
where the last inequality is due to the inclusion relationship $\fcal^\triangle\subset\{f\in\fcal:P_nf^2\leq r\}$.

Taking $g=f-f^\triangle$, then the definition of \tfcal and the fact $f^\triangle\in\fcal$ guarantees that $g\in\tfcal$. Moreover, the construction of $f^\triangle$ implies that $$P_ng^2=\frac{1}{n}\sum_{i=1}^n(f-f^\triangle)^2(X_i)\leq\epsilon^2.$$ Consequently, we have
\begin{align*}
  \sup_{\{f\in\fcal:P_nf^2\leq r\}}\frac{1}{n}\sum_{i=1}^n\sigma_i[f(X_i)-f^\triangle(X_i)]&\leq\sup_{\{g\in\tfcal:P_ng^2\leq\epsilon^2\}}\frac{1}{n}\sum_{i=1}^n\sigma_ig(X_i)=R_n\{f\in\tfcal:P_nf^2\leq\epsilon^2\}.
\end{align*}

Plugging the above inequality into Eq.~\eqref{empirical-local-rademacher-L2-1} gives
\begin{equation}\label{empirical-local-rademacher-L2-2}
  R_n\{f\in\fcal:P_nf^2\leq r\}\leq R_n\{f\in\tfcal:P_nf^2\leq\epsilon^2\}+R_n\{f\in\fcal^\triangle:P_nf^2\leq r\}.
\end{equation}
Taking conditional expectations on both sides of Eq.~\eqref{empirical-local-rademacher-L2-2} and using Lemma~\ref{lem:massart} to bound $\ebb_\sigma R_n\{f\in\fcal^\triangle:P_nf^2\leq r\}$, we derive that
\begin{align*}
  \ebb_\sigma R_n\{f\in\fcal:P_nf^2\leq r\}&\leq\ebb_\sigma R_n\{f\in\tfcal:P_nf^2\leq\epsilon^2\}+\sqrt{\frac{2r\log\ncal(\epsilon/2,\fcal,\|\cdot\|_{L_2(P_n)})}{n}}.
\end{align*}
Since the above inequality holds for any $\epsilon>0$, the desired inequality follows immediately.
\end{proof}
\begin{proof}[Proof of Theorem \ref{thm:local-rademacher-L2}]
  For any $\epsilon>0$ we first fix the sample $X_1,\ldots,X_n$. For any $f\in\fcal$ with $Pf^2\leq r$, there holds that
  \[P_nf^2\leq\sup_{\{f\in\fcal:Pf^2\leq r\}}(P_nf^2-Pf^2)+Pf^2\leq\sup_{\{f\in\fcal:Pf^2\leq r\}}(P_nf^2-Pf^2)+r.\]Consequently, the following result holds almost surely
  \begin{equation}\label{local-rademacher-L2-1}
    \{f\in\fcal:Pf^2\leq r\}\subseteq\left\{f\in\fcal:P_nf^2\leq\sup\nolimits_{\{f\in\fcal:Pf^2\leq r\}}(P_nf^2-Pf^2)+r\right\}.
  \end{equation}
  Using the inclusion relationship~\eqref{local-rademacher-L2-1}, one can control local Rademacher complexities as follows:
  \begin{equation}\label{local-rademacher-L2-2}
  \begin{split}
    &\ebb R_n\{f\in\fcal:Pf^2\leq r\}=\ebb\ebb_\sigma R_n\{f\in\fcal:Pf^2\leq r\}\\
    &\leq\ebb\ebb_\sigma R_n\left\{f\in\fcal:P_nf^2\leq r+\sup\nolimits_{\{f\in\fcal:Pf^2\leq r\}}(P_nf^2-Pf^2)\right\}\\
    &\leq\ebb R_n\{f\in\tfcal:P_nf^2\leq\epsilon^2\}+\sqrt{\frac{2}{n}}\ebb\sqrt{\bigg((r+\sup_{\{f\in\fcal:Pf^2\leq r\}}(P_nf^2-Pf^2)\bigg)\log\ncal(\epsilon/2,\fcal,\|\cdot\|_{L_2(P_n)})}\\
    &\leq\ebb R_n\{f\in\tfcal:P_nf^2\leq\epsilon^2\}+\sqrt{\frac{2\log\ncal(\epsilon/2,\fcal,\|\cdot\|_2)}{n}}\ebb\sqrt{r+\sup_{\{f\in\fcal:Pf^2\leq r\}}(P_nf^2-Pf^2)},
  \end{split}
  \end{equation}
  where the second inequality is a direct corollary of Lemma~\ref{lem:empirical-local-rademacher-L2} and the last inequality follows from Eq. \eqref{metric-capacity}.

  The concavity of $\phi(x)=\sqrt{x}$, coupled with the Jensen inequality, implies that
  \begin{equation}\label{local-rademacher-L2-3}
  \begin{split}
    \ebb\sqrt{r+\sup_{\{f\in\fcal:Pf^2\leq r\}}(P_nf^2-Pf^2)}&\leq\sqrt{r+\ebb\sup_{\{f\in\fcal:Pf^2\leq r\}}(P_nf^2-Pf^2)}\\
    &\leq\sqrt{r+2\ebb R_n\{f^2:f\in\fcal,Pf^2\leq r\}}\\
    &\leq\sqrt{r+4b\ebb R_n\{f\in\fcal:Pf^2\leq r\}},
  \end{split}
  \end{equation}
  where the second inequality follows from the standard symmetrical inequality on Rademacher average \citep[e.g., Lemma A.5]{bartlett2005local} and the third inequality comes from a direct application of Lemma~\ref{lem:contraction inequality} with $\phi(x)=x^2$ (with Lipschitz constant $2b$ on $[-b,b]$).

  Combining Eqs.~\eqref{local-rademacher-L2-2},~\eqref{local-rademacher-L2-3} together, it follows directly that
  \begin{multline*}
    \ebb R_n\{f\in\fcal:Pf^2\leq r\}\leq\ebb R_n\{f\in\tfcal:P_nf^2\leq\epsilon^2\}+\\
    \sqrt{\frac{2\log\ncal(\epsilon/2,\fcal,\|\cdot\|_2)}{n}}\sqrt{r+4b\ebb R_n\{f\in\fcal:Pf^2\leq r\}}.
  \end{multline*}
  Solving the above inequality (a quadratic inequality of $\ebb R_n\{f\in\fcal:Pf^2\leq r\}$) gives that\[\ebb R_n\{f\in\fcal:Pf^2\leq r\}\leq2\ebb R_n\{f\in\tfcal:P_nf^2\leq\epsilon^2\}+\frac{8b\log\ncal(\epsilon/2,\fcal,\|\cdot\|_2)}{n}+\sqrt{\frac{2r\log\ncal(\epsilon/2,\fcal,\|\cdot\|_2)}{n}}.\]
  The proof is complete if we take an infimum over all $\epsilon>0$.
\end{proof}

\subsection{Proofs on explicit local Rademacher complexity bounds}
\begin{proof}[Proof of Corollary \ref{cor:local-rademacher-1}]
  It follows directly from Theorem~\ref{thm:local-rademacher-L2} that
  \begin{equation}\label{cor-local-rademacher-1-1}
    \ebb R_n\{f\in\fcal:Pf^2\leq r\}\leq\inf_{0<\epsilon\leq2\gamma}\lrbrack{2\ebb R_n\{f\in\tfcal:P_nf^2\leq\epsilon^2\}+\frac{8bd\log^p(2\gamma/\epsilon)}{n}+\sqrt{\frac{2rd\log^p(2\gamma/\epsilon)}{n}}},
  \end{equation}
  where \tfcal is defined by Eq.~\eqref{minus-class}. Lemma~\ref{lem:minus-class} and the condition on covering numbers imply that
  \begin{equation}\label{cor-local-rademacher-1-2}
     \log\ncal(\epsilon,\tfcal,\|\cdot\|_2)\leq 2\log\ncal(\epsilon/2,\fcal,\|\cdot\|_2)\leq 2d\log^p(2\gamma/\epsilon),\qquad\text{for any }0<\epsilon\leq2\gamma.
  \end{equation}

  Now one can resort to Lemma~\ref{lem:refined-dudley} to address the term $\ebb R_n\{f\in\tfcal:P_nf^2\leq\epsilon^2\},0<\epsilon<2\gamma$. Indeed, applying Lemma~\ref{lem:refined-dudley} with the assignment $\epsilon_k=2^{-k}\epsilon$ and using the inequality $$\ncal(\epsilon_k,\{f\in\tfcal:P_nf^2\leq\epsilon^2\},\|\cdot\|_{L_2(P_n)})\leq\ncal(\epsilon_k/2,\tfcal,\|\cdot\|_{L_2(P_n)}),$$ the following inequality holds for any $N\in\nbb^+$:
  \begin{equation}\label{cor-local-rademacher-1-3}
  \begin{split}
    \ebb R_n\{f\in\tfcal:P_nf^2\leq\epsilon^2\}&=\ebb\ebb_\sigma R_n\{f\!\in\!\tfcal:P_nf^2\!\leq\!\epsilon^2\}\!\leq\!4\ebb\sum_{k=1}^N\epsilon_{k-1}\sqrt{\frac{\log\ncal(\epsilon_k/2,\tfcal,\|\cdot\|_{L_2(P_n)})}{n}}\!+\!\epsilon_N\\
    &\leq 2^{7/2}\sqrt{\frac{d}{n}}\epsilon\sum_{k=1}^N2^{-k}\log^{p/2}\lrgroup{2^{k+2}\gamma \epsilon^{-1}}+\epsilon_N\qquad(\text{according to Eq.~\eqref{cor-local-rademacher-1-2}})\\
    &\leq 2^{(7+p)/2}\sqrt{\frac{d}{n}}\epsilon\sum_{k=1}^N2^{-k}\lrbrack{\big((k+1)\log2\big)^{p/2}+\log^{p/2}(2\gamma\epsilon^{-1})}+\epsilon_N\\
    &\leq 2^{(7+p)/2}\sqrt{\frac{d}{n}}\epsilon\lrbrack{c(p)+\log^{p/2}(2\gamma/\epsilon)}+\epsilon_N,
  \end{split}
  \end{equation}
  where the third inequality follows from the standard result $(a+b)^{p/2}\leq\lrbrack{2\max(a,b)}^{p/2}\leq 2^{p/2}(a^{p/2}+b^{p/2}),a,b\geq0$ and the last inequality is due to the fact $\sum_{k=1}^N2^{-k}\big((k+1)\log2\big)^{p/2}<\infty$.

  Letting $N\to\infty$ in Eq.~\eqref{cor-local-rademacher-1-3} and noticing Eq.~\eqref{cor-local-rademacher-1-1}, one derives that

  \begin{equation}\label{cor-local-rademacher-1-4}
  \begin{split}
  \ebb R_n\{f\in\fcal:Pf^2\leq r\}&\leq\inf_{0<\epsilon\leq2\gamma}\lrbrack{2^{(9+p)/2}\sqrt{\frac{d}{n}}\epsilon\Big(c(p)\!+\!\log^{p/2}(2\gamma/\epsilon)\Big)\!+\!\frac{8bd\log^p(2\gamma/\epsilon)}{n}\!+\!\sqrt{\frac{2rd\log^p(2\gamma/\epsilon)}{n}}}\\
  &\leq c(b,p,\gamma)\inf_{0<\epsilon\leq\gamma}\lrbrack{\sqrt{\frac{d}{n}}\epsilon\log^{p/2}(2\gamma/\epsilon)+\frac{d\log^p(2\gamma/\epsilon)}{n}+\sqrt{\frac{rd\log^p(2\gamma/\epsilon)}{n}}}.
  \end{split}
  \end{equation}

  Taking the choice $\epsilon=\sqrt{r}$ in Eq.~\eqref{cor-local-rademacher-1-4}, there holds that\[\ebb R_n\{f\in\fcal:Pf^2\leq r\}\leq c(b,p,\gamma)\lrbrack{\sqrt{\frac{dr\log^p(2\gamma r^{-1/2})}{n}}+\frac{d\log^p(2\gamma r^{-1/2})}{n}}.\]

  Taking the assignment $\epsilon=n^{-1/2}$, we derive that\[\ebb R_n\{f\in\fcal:Pf^2\leq r\}\leq c(b,p,\gamma)\lrbrack{\frac{d\log^p(2\gamma n^{1/2})}{n}+\sqrt{\frac{rd\log^p(2\gamma n^{1/2})}{n}}}.\]Since $\ebb R_n\{f\in\fcal:Pf^2\leq r\}$ can be upper bounded for any $0<\epsilon\leq\gamma$, the desired inequality is immediate.
\end{proof}
\begin{proof}[Proof of Corollary \ref{cor:local-rademacher-3}]
  Theorem~\ref{thm:local-rademacher-L2} can be applied here to show that
  \begin{equation}\label{cor-local-rademacher-3-1}
    \ebb R_n\{f\in\fcal:Pf^2\leq r\}\leq\inf_{\epsilon>0}\lrbrack{2\ebb R_n\{f\in\tfcal:P_nf^2\leq\epsilon^2\}+\frac{8b\gamma\epsilon^{-p}2^p\log^2\frac{4}{\epsilon}}{n}+\sqrt{\frac{2r\gamma\epsilon^{-p}2^p\log^2\frac{4}{\epsilon}}{n}}}.
  \end{equation}
  Lemma~\ref{lem:minus-class} gives the following entropy condition for \tfcal:
  \begin{equation}\label{cor-local-rademacher-3-2}
    \log\ncal(\epsilon,\tfcal,\|\cdot\|_2)\leq2\log\ncal(\epsilon/2,\fcal,\|\cdot\|_2)\leq2^{p+1}\gamma\epsilon^{-p}\log^2\frac{4}{\epsilon}.
  \end{equation}
  Now applying Lemma~\ref{lem:refined-dudley} with the assignment $\epsilon_k=2^{-k}\epsilon$ and analyzing analogously to the proof of Corollary~\ref{cor:local-rademacher-1} except using the entropy condition~\eqref{cor-local-rademacher-3-2}, one derives that
  \begin{equation}\label{cor-local-rademacher-3-3}
  \begin{split}
    \ebb R_n\{f\in\tfcal:P_nf^2\leq\epsilon^2\}&\leq4\ebb\sum_{k=1}^N\epsilon_{k-1}\sqrt{\frac{\log\ncal(\epsilon_k/2,\tfcal,\|\cdot\|_{L_2(P_n)})}{n}}+\epsilon_N\\
    &\leq4\sum_{k=1}^N2^{1-k}\epsilon\sqrt{\frac{\gamma\epsilon^{-p}2^{(k+2)p+1}\log^2\frac{2^{k+3}}{\epsilon}}{n}}+2^{-N}\epsilon\\
    &=\sqrt{\frac{\gamma}{n}}2^{7/2+p}\epsilon^{1-p/2}\sum_{k=1}^N2^{k(p-2)/2}\big[\log\frac{1}{\epsilon}+(k+3)\log2\big]+2^{-N}\epsilon.
  \end{split}
  \end{equation}

  We now continue our discussion by distinguishing three cases according to the magnitude of $p$:
  \begin{enumerate}[(a)]
    \item \textsc{case $0<p<2$.} In this case, the series $\sum_{k=1}^\infty 2^{k(p-2)/2}[\log\frac{1}{\epsilon}+(k+3)\log2]$ converges and thus one can tend $N\to\infty$ in Eq.~\eqref{cor-local-rademacher-3-3} to derive the bound $\ebb R_n\{f\in\tfcal:P_nf^2\leq\epsilon^2\}\leq cn^{-1/2}\epsilon^{1-p/2}\log\frac{1}{\epsilon}$. Plugging this inequality back into Eq.~\eqref{cor-local-rademacher-3-1} one obtains that
        $$
          \ebb R_n\{f\in\fcal:Pf^2\leq r\}\leq c\inf_{\epsilon>0}\lrbrack{n^{-1/2}\epsilon^{1-p/2}\log\frac{1}{\epsilon}+\epsilon^{-p}n^{-1}\log^2\frac{4}{\epsilon}+\sqrt{r\epsilon^{-p}n^{-1}\log^2\frac{4}{\epsilon}}}.
        $$

    \item \textsc{case $p=2$.} For this particular $p$, Eq. \eqref{cor-local-rademacher-3-1} and Eq. \eqref{cor-local-rademacher-3-3} imply that
        \begin{align*}
        \hspace*{-0.6cm}\ebb R_n\{f\in\fcal:Pf^2\leq r\}&\leq c\inf_{\epsilon>0}\inf_{N\in\nbb^+}\lrbrack{n^{-1/2}(N\log\frac{1}{\epsilon}+N^2)+2^{-N}\epsilon+\epsilon^{-2}n^{-1}\log^2\frac{4}{\epsilon}+\sqrt{r\epsilon^{-2}n^{-1}\log^2\frac{4}{\epsilon}}}\\
          &\leq c\big[n^{-1/2}\log^2n+\sqrt{rn^{-1}}\big],
        \end{align*}
        where in the last step we simply take the choice $\epsilon=1$ and $N=\lrceil{2^{-1}\log_2n}$.
    \item \textsc{case $p>2$.} In this case, taking the choice $\epsilon=1$ in Eqs. \eqref{cor-local-rademacher-3-1}, \eqref{cor-local-rademacher-3-3} we have
    \begin{align*}
    \ebb R_n\{f\in\fcal:Pf^2\leq r\}&\leq c\inf_{N\in\nbb^+}\big[n^{-1/2}\sum_{k=1}^{N}(k+3)2^{k(p-2)/2}+2^{-N}+n^{-1}+\sqrt{rn^{-1}}\big]\\
    &\leq c\inf_{N\in\nbb^+}\big[n^{-1/2}N2^{N(p-2)/2}+2^{-N}+\sqrt{rn^{-1}}\big]\\
    &\leq c[n^{-1/p}\log n+\sqrt{rn^{-1}}],
    \end{align*}
    where we choose $N=\lceil p^{-1}\log_2n\rceil$ in the last step.
  \end{enumerate}
\end{proof}

Using a similar deduction strategy, one can also prove Corollary \ref{cor:local-rademacher-2} on local Rademacher complexity bounds when the entropy number grows as a polynomial of $1/\epsilon$. For simplicity we omit the proof here.

\subsection{Proofs on generalization analysis}
\begin{proof}[Proof of Theorem \ref{thm:application-1}]
  We consider the functional $T(f):=Pf^2$ here. The structural result on covering numbers implies that~\citep{mendelson2003few}
  \[\log\ncal(\epsilon,\fcal,\|\cdot\|_2)\leq\log\ncal(\epsilon/L,\hcal,\|\cdot\|_2)\leq d\log^p(\gamma L/\epsilon).\]
  Corollary~\ref{cor:local-rademacher-1} implies that
  $$
    \psi(r):=c\lrbrack{\frac{d\log^p(2\gamma n^{1/2})}{n}+\sqrt{\frac{rd\log^p(2\gamma n^{1/2})}{n}}}
  $$
  is an appropriate choice meeting the condition of Lemma~\ref{thm:bartlett}. Let $r^*$ be its fixed point then we know that
  \[r^*= c\lrbrack{\frac{d\log^p(2\gamma n^{1/2})}{n}+\sqrt{\frac{r^*d\log^p(2\gamma n^{1/2})}{n}}}.\]Solving this equality gives $r^*\leq cdn^{-1}\log^p(n)$.
  It can be directly checked that any $f\in\fcal$ also satisfies $\|f\|_\infty\leq 4b^2$. Consequently, one can apply Lemma~\ref{thm:bartlett} here to show that for the particular function $\hat{f}_n=\ell(\hat{h}_n(x),y)-\ell(h^*(x),y)$, the following inequality holds with probability at least $1-\delta$
\begin{align*}
  P\hat{f}_n&\leq\frac{K}{K-1}P_n\hat{f}_n+\frac{704Kcd\log^pn}{Bn}+\frac{\log(1/\delta)(88b^2+416b^2K)}{n},\qquad\forall K>1.
\end{align*}
Using the above inequality and the fact $P_n\hat{f}_n=\emp(\hat{h}_n)-\emp(h^*)\leq0$, we immediately derive the desired result.
\end{proof}
\begin{proof}[Proof of Theorem \ref{thm:application-2}]
Let $\epsilon$ be a positive number to be fixed later. The entropy assumption imply that $\log\ncal(\epsilon,\fcal,\|\cdot\|_2)\leq c\epsilon^{-p}\log^2\frac{1}{\epsilon}$, from which Corollary~\ref{cor:local-rademacher-3} implies that
$$
  \psi_\epsilon(r):=c\lrbrack{n^{-1/2}\epsilon^{1-p/2}\log\frac{1}{\epsilon}+\epsilon^{-p}n^{-1}\log^2\frac{4}{\epsilon}+\sqrt{r\epsilon^{-p}n^{-1}\log^2\frac{4}{\epsilon}}}
$$ is a function meeting the condition of Lemma \ref{thm:bartlett}. The associated fixed point $r_\epsilon^*=\psi(r_\epsilon^*)$ satisfies the constraint
$$
  r_\epsilon^*\leq c\big[n^{-\frac{1}{2}}\epsilon^{1-\frac{p}{2}}\log\frac{1}{\epsilon}+\epsilon^{-p}n^{-1}\log^2\frac{4}{\epsilon}\big].
$$
For the specific choice $\epsilon_0=(\log n)^{\frac{2}{p+2}}n^{-\frac{1}{p+2}}$ we get $r_{\epsilon_0}^*=cn^{-\frac{2}{p+2}}(\log n)^{\frac{2-p}{p+2}}\log\frac{n}{(\log n)^{\frac{2}{p+2}}}$. Plugging this bound on $r^*_{\epsilon_0}$ into Lemma \ref{thm:bartlett} completes the proof.
\end{proof}

\section{Conclusions}\label{sec:conclusion}
This paper provides a systematic approach to estimating local Rademacher complexities with covering numbers. Local Rademacher complexity is an effective concept in learning theory and has recently received increasing attention since it captures the property that the prediction rule picked by a learning algorithm always lies in a subset of the original class. We provide a general local Rademacher complexity bound, which captures in an elegant form to relate the complexities with constraint on the $L_2(P)$ norm to the corresponding ones with constraint on the $L_2(P_n)$ norm. This bound is convenient to calculate and is easily applicable to practical learning problems. We show that our general result (Theorem~\ref{thm:local-rademacher-L2}) could yield local Rademacher complexity bounds superior to that in~\citet{mendelson2002improving,mendelson2003few}, when applied to function classes satisfying general entropy conditions. We also apply the derived local Rademacher complexity bounds to the generalization analysis. 

\section*{Acknowledgement}
The work is partially supported by Science Computing and Intelligent Information Processing of GuangXi higher education key laboratory (Grant No. GXSCIIP201409).
\appendix
\numberwithin{equation}{section}
\numberwithin{theorem}{section}
\numberwithin{figure}{section}
\numberwithin{table}{section}
\renewcommand{\thesection}{{\Alph{section}}}
\renewcommand{\thesubsection}{\Alph{section}.\arabic{subsection}}
\renewcommand{\thesubsubsection}{\Roman{section}.\arabic{subsection}.\arabic{subsubsection}}
\section{Lemmas}
Lemma \ref{lem:massart} presents effective empirical complexity bounds for function classes of finite cardinality.
\begin{lemma}[Massart lemma~\citep{bousquet2002concentration}]\label{lem:massart}
  Suppose that \fcal is a finite class with cardinality $N$, then the empirical local Rademacher complexity can be bounded as follows:\[\ebb_\sigma R_n\{f\in\fcal:P_nf^2\leq r\}\leq\sqrt{\frac{2r\log N}{n}}.\]
\end{lemma}

\begin{lemma}[\citep{pollard1984convergence}]\label{lem:minus-class}
  Let $\|\cdot\|$ be a norm defined on the class \fcal. If \tfcal is defined by Eq.~\eqref{minus-class}, then we have $\ncal(\epsilon,\tfcal,\|\cdot\|)\leq\ncal^2(\epsilon/2,\fcal,\|\cdot\|)$.
\end{lemma}

Since our definition of covering numbers requires the $\epsilon$-cover to belong to the original class, covering numbers of a sub-class is not necessarily smaller than that of the whole class. However, we have the following structural result for tackling covering numbers of a sub-class.

\begin{lemma}[\citep{pollard1984convergence}]\label{lem:covering-number-relationship}
  Let $\fcal$ be a class of functions from \xcal to \rbb and let $\fcal_0\subseteq\fcal$ be a subset. Then for any $\epsilon>0$, we have the following relationship on covering numbers: $\ncal(\epsilon,\fcal_0,d)\leq\ncal(\epsilon/2,\fcal,d)$.
\end{lemma}

The following structural result on Rademacher complexities provides us a powerful tool to tackle the complexity of a composite class via that of the basis class.
\begin{lemma}[Contraction property~\citep{bartlett2005local}]\label{lem:contraction inequality}
  Let $\phi$ be a Lipschitz function with constant $L$, that is, $|\phi(x)-\phi(y)|\leq L|x-y|$. Then for every function class \fcal there holds
  \begin{equation}\label{comprison-inequality}
    \ebb_\sigma R_n\phi\circ\fcal\leq L\ebb_\sigma R_n\fcal,
  \end{equation}
  where $\phi\circ\fcal:=\{\phi\circ f:f\in\fcal\}$ and $\circ$ is the composition operator.
\end{lemma}

\begin{lemma}[Refined entropy integral~\citep{mendelson2002improving}]\label{lem:refined-dudley}
  Let $X_1,\ldots,X_n$ be a sequence of examples and let $P_n$ be the associated empirical measure. For any function class \fcal and any monotone sequence $(\epsilon_k)_{k=0}^\infty$ decreasing to $0$ such that $\epsilon_0\geq\sup_{f\in\fcal}\sqrt{P_nf^2}$, the following inequality holds for every non-negative integer $N$:
  \begin{equation}\label{refined-dudley-integral}
    \ebb_\sigma R_n\fcal\leq 4\sum_{k=1}^N\epsilon_{k-1}\sqrt{\frac{\log\ncal(\epsilon_k,\fcal,\|\cdot\|_{L_2(P_n)})}{n}}+\epsilon_N.
  \end{equation}
\end{lemma}


\bibliographystyle{abbrvnat}
\bibliography{LRC}

\begin{thebibliography}{34}
\providecommand{\natexlab}[1]{#1}
\providecommand{\url}[1]{\texttt{#1}}
\expandafter\ifx\csname urlstyle\endcsname\relax
  \providecommand{\doi}[1]{doi: #1}\else
  \providecommand{\doi}{doi: \begingroup \urlstyle{rm}\Url}\fi

\bibitem[Bartlett and Mendelson(2002)]{bartlett2002rademacher}
P.~Bartlett and S.~Mendelson.
\newblock Rademacher and gaussian complexities: Risk bounds and structural
  results.
\newblock \emph{J. Mach. Learn. Res.}, 3:\penalty0 463--482, 2002.

\bibitem[Bartlett et~al.(2005)Bartlett, Bousquet, and
  Mendelson]{bartlett2005local}
P.~Bartlett, O.~Bousquet, and S.~Mendelson.
\newblock Local rademacher complexities.
\newblock \emph{Ann. Stat.}, 33\penalty0 (4):\penalty0 1497--1537, 2005.

\bibitem[Blanchard et~al.(2008)Blanchard, Bousquet, and
  Massart]{blanchard2008statistical}
G.~Blanchard, O.~Bousquet, and P.~Massart.
\newblock Statistical performance of support vector machines.
\newblock \emph{Ann. Stat.}, 36\penalty0 (2):\penalty0 489--531, 2008.

\bibitem[Bousquet(2002)]{bousquet2002concentration}
O.~Bousquet.
\newblock \emph{Concentration Inequalities and Empirical Processes Theory
  Applied to the Analysis of Learning Algorithms}.
\newblock PhD thesis, Ecole Polytechnique, Paris, 2002.

\bibitem[Bousquet(2003)]{bousquet2003new}
O.~Bousquet.
\newblock New approaches to statistical learning theory.
\newblock \emph{Ann. Inst. Stat. Math.}, 55\penalty0 (2):\penalty0 371--389,
  2003.

\bibitem[Chen et~al.(2014)Chen, Peng, Zhou, Li, and Pan]{chen2014extreme}
H.~Chen, J.~Peng, Y.~Zhou, L.~Li, and Z.~Pan.
\newblock Extreme learning machine for ranking: Generalization analysis and
  applications.
\newblock \emph{Neural Networks}, 53:\penalty0 119--126, 2014.

\bibitem[Cortes et~al.(2013)Cortes, Kloft, and Mohri]{cortes2013learning}
C.~Cortes, M.~Kloft, and M.~Mohri.
\newblock Learning kernels using local rademacher complexity.
\newblock In \emph{Advances in Neural Information Processing Systems}, pages
  2760--2768, 2013.

\bibitem[Cucker and Smale(2002)]{cucker2002mathematical}
F.~Cucker and S.~Smale.
\newblock On the mathematical foundations of learning.
\newblock \emph{Bull. Am. Math. Soc.}, 39\penalty0 (1):\penalty0 1--50, 2002.

\bibitem[Cucker and Zhou(2007)]{cucker2007learning}
F.~Cucker and D.-X. Zhou.
\newblock \emph{Learning theory: an approximation theory viewpoint}.
\newblock Cambridge Univ. Press, Cambridge, 2007.

\bibitem[Dudley(1967)]{dudley1967sizes}
R.~Dudley.
\newblock The sizes of compact subsets of hilbert space and continuity of
  gaussian processes.
\newblock \emph{J. Funct. Anal}, 1\penalty0 (3):\penalty0 290--330, 1967.

\bibitem[Hastie et~al.(2001)Hastie, Tibshirani, and
  Friedman]{hastie2001elements}
T.~Hastie, R.~Tibshirani, and J.~Friedman.
\newblock \emph{The elements of statistical learning: data mining, inference,
  and prediction}.
\newblock Springer-Verlag, New York, 2001.

\bibitem[Kloft and Blanchard(2011)]{kloft2011local}
M.~Kloft and G.~Blanchard.
\newblock The local rademacher complexity of lp-norm multiple kernel learning.
\newblock In \emph{Advances in Neural Information Processing Systems}, pages
  2438--2446, 2011.

\bibitem[Kloft and Blanchard(2012)]{kloft2012convergence}
M.~Kloft and G.~Blanchard.
\newblock On the convergence rate of lp-norm multiple kernel learning.
\newblock \emph{J. Mach. Learn. Res.}, 13\penalty0 (1):\penalty0 2465--2502,
  2012.

\bibitem[Kolmogorov and Tikhomirov(1959)]{kolmogorov1959varepsilon}
A.~N. Kolmogorov and V.~M. Tikhomirov.
\newblock $\varepsilon$-entropy and $\varepsilon$-capacity of sets in function
  spaces.
\newblock \emph{Uspekhi Matematicheskikh Nauk}, 14\penalty0 (2):\penalty0
  3--86, 1959.

\bibitem[Koltchinskii(2001)]{koltchinskii2001rademacher}
V.~Koltchinskii.
\newblock Rademacher penalties and structural risk minimization.
\newblock \emph{IEEE Trans. Inf. Theory}, 47\penalty0 (5):\penalty0 1902--1914,
  2001.

\bibitem[Koltchinskii and Panchenko(2000)]{koltchinskii2000rademacher}
V.~Koltchinskii and D.~Panchenko.
\newblock Rademacher processes and bounding the risk of function learning.
\newblock In E.~Gin\'e, D.~Mason, and J.~Wellner, editors, \emph{Hign
  Dimensional Probability {\uppercase\expandafter{\romannumeral2}}}, pages
  443--458, Boston, 2000. Birkh\"auser.

\bibitem[Ledoux and Talagrand(1991)]{ledoux1991probability}
M.~Ledoux and M.~Talagrand.
\newblock \emph{Probability in Banach Spaces: isoperimetry and processes}.
\newblock Springer-Verlag, Berlin, 1991.

\bibitem[Lei and Ding(2014)]{lei2014refined}
Y.~Lei and L.~Ding.
\newblock Refined {R}ademacher chaos complexity bounds with applications to the
  multikernel learning problem.
\newblock \emph{Neural. Comput.}, 26\penalty0 (4):\penalty0 739--760, 2014.

\bibitem[Lei et~al.(2015{\natexlab{a}})Lei, Ding, and
  Zhang]{lei2015generalization}
Y.~Lei, L.~Ding, and W.~Zhang.
\newblock Generalization performance of radial basis function networks.
\newblock \emph{IEEE Transactions on Neural Networks and Learning Systems},
  26\penalty0 (3):\penalty0 551--564, 2015{\natexlab{a}}.

\bibitem[Lei et~al.(2015{\natexlab{b}})Lei, Dogan, Binder, and
  Kloft]{lei2015multi}
Y.~Lei, {\"U}.~Dogan, A.~Binder, and M.~Kloft.
\newblock Multi-class svms: From tighter data-dependent generalization bounds
  to novel algorithms.
\newblock \emph{Advances in Neural Information Processing Systems, To appear},
  2015{\natexlab{b}}.

\bibitem[Lv and Zhou(2015)]{lv2015optimal}
S.~Lv and F.~Zhou.
\newblock Optimal learning rates of lp-type multiple kernel learning under
  general conditions.
\newblock \emph{Information Sciences}, 294:\penalty0 255--268, 2015.

\bibitem[Massart(2000)]{massart2000some}
P.~Massart.
\newblock Some applications of concentration inequalities to statistics.
\newblock \emph{Annales de la facult{\'e} des sciences de Toulouse}, 9\penalty0
  (2):\penalty0 245--303, 2000.

\bibitem[Mendelson(2002)]{mendelson2002improving}
S.~Mendelson.
\newblock Improving the sample complexity using global data.
\newblock \emph{IEEE Trans. Inf. Theory}, 48\penalty0 (7):\penalty0 1977--1991,
  2002.

\bibitem[Mendelson(2003{\natexlab{a}})]{mendelson2003few}
S.~Mendelson.
\newblock A few notes on statistical learning theory.
\newblock In S.~Mendelson and A.~Smola, editors, \emph{Advanced Lectures on
  Machine Learning. Lect. Notes Comput. Sci. 2600}, pages 1--40.
  Springer-Verlag, Berlin, 2003{\natexlab{a}}.

\bibitem[Mendelson(2003{\natexlab{b}})]{mendelson2003performance}
S.~Mendelson.
\newblock On the performance of kernel classes.
\newblock \emph{J. Mach. Learn. Res.}, 4:\penalty0 759--771,
  2003{\natexlab{b}}.

\bibitem[Oneto et~al.(2015)Oneto, Ghio, Ridella, and Anguita]{oneto2015local}
L.~Oneto, A.~Ghio, S.~Ridella, and D.~Anguita.
\newblock Local rademacher complexity: Sharper risk bounds with and without
  unlabeled samples.
\newblock \emph{Neural Networks}, 65:\penalty0 115--125, 2015.

\bibitem[Pollard(1984)]{pollard1984convergence}
D.~Pollard.
\newblock \emph{Convergence of stochastic processes}.
\newblock Springer-Verlag, New York, 1984.

\bibitem[Srebro et~al.(2010)Srebro, Sridharan, and
  Tewari]{srebro2010optimistic}
N.~Srebro, K.~Sridharan, and A.~Tewari.
\newblock Optimistic rates for learning with a smooth loss.
\newblock \emph{arXiv preprint arXiv:1009.3896}, 2010.

\bibitem[Vapnik(2000)]{vapnik2000nature}
V.~Vapnik.
\newblock \emph{The nature of statistical learning theory}.
\newblock Springer-Verlag, New York, 2000.

\bibitem[Vapnik and Chervonenkis(1971)]{vapnik1971uniform}
V.~Vapnik and A.~Chervonenkis.
\newblock On the uniform convergence of relative frequencies of events to their
  probabilities.
\newblock \emph{Theory Probab. Appl.}, 16\penalty0 (2):\penalty0 264--280,
  1971.

\bibitem[Wu et~al.(2006)Wu, Ying, and Zhou]{wu2006learning}
Q.~Wu, Y.~Ying, and D.-X. Zhou.
\newblock Learning rates of least-square regularized regression.
\newblock \emph{Foundations of Computational Mathematics}, 6\penalty0
  (2):\penalty0 171--192, 2006.

\bibitem[Ying and Campbell(2010)]{ying2010rademacher}
Y.~Ying and C.~Campbell.
\newblock Rademacher chaos complexities for learning the kernel problem.
\newblock \emph{Neural. Comput.}, 22\penalty0 (11):\penalty0 2858--2886, 2010.

\bibitem[Zhou(2002)]{zhou2002covering}
D.-X. Zhou.
\newblock The covering number in learning theory.
\newblock \emph{J. Complex.}, 18\penalty0 (3):\penalty0 739--767, 2002.

\bibitem[Zhou(2003)]{zhou2003capacity}
D.-X. Zhou.
\newblock Capacity of reproducing kernel spaces in learning theory.
\newblock \emph{IEEE Trans. Inf. Theory}, 49\penalty0 (7):\penalty0 1743--1752,
  2003.

\end{thebibliography}

\end{document}